\documentclass[10pt]{article}

\usepackage[preprint]{tmlr}
\usepackage{amsmath,amssymb,amsthm}
\usepackage{booktabs}
\usepackage{graphicx}
\usepackage{subcaption}
\usepackage{float}
\usepackage{placeins}
\usepackage{enumitem}
\usepackage{multirow}
\usepackage{tabularx}
\usepackage{xcolor}
\usepackage{xspace}
\usepackage[hidelinks]{hyperref}
\usepackage{url}
\usepackage{arydshln}

\definecolor{sageblue}{RGB}{44,110,155}
\newcommand{\sagecp}{\textsc{SAGE-CP}\xspace}
\newcommand{\caldev}{\mathcal{D}_{\mathrm{dev}}}
\newcommand{\calcal}{\mathcal{D}_{\mathrm{cal}}}
\newcommand{\labels}{\mathcal{Y}}
\newcommand{\mincov}{\mathrm{MinCov}}
\newcommand{\takeaway}[1]{\par\smallskip\noindent\textbf{Takeaway.} #1\par\smallskip}
\newtheorem{theorem}{Theorem}
\newtheorem{proposition}[theorem]{Proposition}
\newtheorem{corollary}[theorem]{Corollary}
\newfloat{algorithm}{tbp}{loa}
\floatname{algorithm}{Algorithm}

\title{SAGE-CP: When Marginal Coverage Hides\\
Route-Conditional Failures in Adaptive VLM Ensembles}
\author{%
\centering
\name Mushir Akhtar \quad M. Tanveer \\
\addr Department of Mathematics, Indian Institute of Technology Indore, India \\
\normalfont\small
\texttt{mushirakhtar.ml@gmail.com} \quad \texttt{mtanveer@iiti.ac.in}}

\hypersetup{
  pdftitle={SAGE-CP: When Marginal Coverage Hides Route-Conditional Failures in Adaptive VLM Ensembles},
  pdfauthor={Mushir Akhtar and M. Tanveer}
}

\renewcommand{\headrulewidth}{0pt}
\fancypagestyle{firstpagepreprint}{%
  \fancyhf{}%
  \fancyfoot[L]{\small Preprint}%
  \fancyfoot[C]{\small\thepage}%
  \renewcommand{\headrulewidth}{0pt}%
  \renewcommand{\footrulewidth}{0pt}%
}

\begin{document}
\raggedbottom
\maketitle
\thispagestyle{firstpagepreprint}

\begin{abstract}
Adaptive vision-language model (VLM) ensembles save computation by stopping
early on easy images and escalating harder ones. Although a single conformal
threshold can make the final predictor marginally valid, it can still fail
severely along one of its computational routes. We characterize this failure precisely: if a route has prevalence \(\pi\), marginal \(1-\alpha\) coverage allows its conditional coverage to fall to \([1-\alpha/\pi]_+\).
To repair the mismatch, we learn the router on development data, freeze every
routing choice, and calibrate the returned score separately for early exits and
escalated inputs. We call this pipeline \emph{Sequential Adaptive
Group-Conditional Ensemble Conformal Prediction} (\sagecp). Its calibration
step is standard Mondrian split conformal prediction applied to the operational
groups created by the router. Under exchangeability, the resulting sets have
finite-sample coverage within each route.
Across ten benchmarks, four frozen VLMs, and five aligned splits, both global
calibration and \sagecp achieve $0.900$ marginal coverage at a $0.90$ target,
but their route behavior differs substantially. \sagecp raises mean
worst-route coverage from $0.623$ to $0.887$ and reduces the mean route gap from
$0.324$ to $0.025$, using the same routed scores and $1.177$ source calls per
image. The stronger guarantee increases macro mean set size from $1.893$ to
$2.212$; the sets nevertheless remain smaller than those of the full ensemble
($2.486$) while using $70.6\%$ fewer source calls. Sets grow on six datasets,
with the largest change from $1.044$ to $2.623$ labels on Flowers-102. The
route-reliability result persists across six fixed query fractions, while finer
exit-depth groups reveal the expected need for more calibration data on rare
routes. Finally, across five datasets, \sagecp reduces measured batch-one
latency by $58.8\%$--$63.4\%$ and gross GPU energy
by $55.7\%$--$57.5\%$ relative to full inference. Thus, routing provides the
computational savings, while route-stratified calibration attaches coverage to
the operational paths created by the router.
\end{abstract}

\section{Introduction}
\label{sec:introduction}

Vision-language models (VLMs) such as CLIP support zero-shot image
classification by comparing an image representation with text representations
of candidate labels \citep{radford2021learning}. Their behavior depends on the
visual backbone, pretraining corpus, and scaling strategy
\citep{cherti2023openclip,gadre2023datacomp,tschannen2025siglip2}. Consequently,
different frozen VLMs need not make the same errors. Combining their evidence
can produce a stronger predictor, but evaluating every source for every image
also fixes inference cost at that of the full ensemble.

Adaptive acquisition offers a more flexible operating point. A system can
evaluate a strong source first, stop when the evidence is sufficiently clear,
and acquire additional sources only for ambiguous inputs. Related work uses
score aggregation to conformalize multiple predictors
\citep{ochoa2024ensemble,luo2025averaging,alami2026symmetric}, while cascades and
early-exit methods reduce computation by allocating different amounts of
inference to different inputs
\citep{fisch2020cascade,jazbec2024early,dou2026cascade}. Combining these ideas
creates an adaptive ensemble in which an image's output and computational path
are determined jointly. Some images exit after one source, whereas others are
escalated through a longer sequence.

Conformal prediction (CP) can wrap a fixed predictor with a label set whose
marginal coverage is at least $1-\alpha$ under exchangeability
\citep{vovk2005algorithmic,angelopoulos2023gentle}. Work on conformal VLMs has
studied the effects of backbone and score choice, transductive adaptation,
localized thresholds, score variability, and test-time augmentation
\citep{fillioux2026are,silva2025zeroshot,fuchs2026localized,
zeng2026empirical,shanmugam2025tta}. These advances improve how a fixed VLM or
ensemble is calibrated and how efficiently it represents uncertainty. An
adaptive ensemble introduces a different question: does a marginal guarantee
remain informative for each population created by its routing policy?

The distinction matters because marginal coverage averages over the routing
decision. A system can overcover images that stop early while concentrating
most errors among escalated images, or it can fail in the opposite direction.
Although unrestricted pointwise conditional coverage is unavailable without
additional assumptions \citep{barber2021limits}, finite, prespecified groups can
receive meaningful conformal guarantees
\citep{vovk2012conditional,gibbs2025conditional,martinezgil2024groups,
bairaktari2025kandinsky}. The computational routes of a frozen adaptive
predictor provide such groups: they are observable without the test label and
correspond directly to the acquisition paths experienced at deployment.

Marginal validity can nevertheless hide an extreme route failure. We prove
that a route with prevalence at most $\alpha$ may have zero coverage while the
complete predictor remains marginally valid. The experiments show that this
possibility is practically relevant. At a $0.90$ target, global calibration
achieves $0.900$ mean marginal coverage across ten datasets, yet its mean
minimum-route coverage is only $0.623$. On Food-101, Oxford Pets, and
Caltech-101, the minimum falls to approximately $0.30$--$0.37$.

We address this mismatch by aligning calibration with the deployed computation.
We learn the complete routing policy on independent development data and freeze
all design choices before final calibration. Since route membership under the
frozen policy is observable without labels, standard Mondrian split CP can
calibrate the returned score separately within each route. We call the resulting
pipeline \emph{Sequential Adaptive Group-Conditional Ensemble Conformal
Prediction} (\sagecp).

This repair changes the calibration boundary rather than the predictor.
\sagecp does not retrain a VLM, alter the routed probabilities, or require
additional inference beyond the adaptive policy. Routing supplies the source,
latency, and energy savings, while route-stratified calibration attaches
coverage to the operational populations created by that routing. We formalize
the necessary separation between development and calibration and prove that a
finite routing policy learned on independent development data remains valid
after it is frozen.

Routes are therefore more than an arbitrary post-hoc taxonomy. A route decides
which sources are acquired and which score is returned, coupling group
membership to the computation itself. Another fixed partition can validly
protect a different declared population, but it need not protect the early-exit
and escalated populations. Our target is specifically operational: the
guarantee should hold along the acquisition paths that users experience.

Our experiments are organized around four research questions (RQs). \textbf{RQ1}
examines whether marginal coverage can conceal route-conditional failures.
\textbf{RQ2} evaluates whether route-specific calibration repairs these failures
without requiring full-ensemble inference. \textbf{RQ3} tests whether the
findings remain consistent across datasets, risk levels, calibration sizes,
route granularities, and routing policies. Finally, \textbf{RQ4} determines
whether fewer source evaluations translate into measurable reductions in
latency and energy consumption.

Our contributions are:
\begin{enumerate}[leftmargin=*,itemsep=2pt]

    \item A sharp prevalence-dependent characterization of what marginal
    conformal validity can hide. It quantifies the worst attainable coverage
    within a route and the largest attainable coverage gap between two routes
    as functions of route prevalence.

    \item A calibration-safe adaptive-VLM protocol based on Mondrian split
    conformal prediction, together with a finite-sample route-conditional
    coverage theorem. The finite routing policy, conformity score, source
    order, and route partition may be selected using independent development
    data and then frozen before calibration.

    \item A controlled five-split evaluation across ten datasets and four
    heterogeneous VLMs. At the same mean marginal coverage ($.900$), routed
    scores, and source usage, \sagecp increases mean minimum-route coverage
    from $.623$ to $.887$ and reduces the mean route-coverage gap from $.324$
    to $.025$. Descriptive paired stability intervals exclude zero in favor
    of \sagecp on eight datasets; on the remaining two, they include zero,
    indicating no consistent difference across splits.

    \item A real batch-one sequential deployment study across five datasets.
    Relative to full-ensemble inference, \sagecp reduces measured end-to-end
    latency by $58.8\%$--$63.4\%$ and sampled gross GPU energy by
    $55.7\%$--$57.5\%$. The measurements include image loading and decoding,
    preprocessing, device transfer, synchronized GPU inference, routing
    overhead, and peak memory.

\end{enumerate}

The remainder of the paper is organized as follows. Section~\ref{sec:related} positions the work within conformal classification and
adaptive inference. Sections~\ref{sec:problem} and~\ref{sec:method} formalize the
route-coverage failure and present \sagecp. Sections~\ref{sec:experiments}
and~\ref{sec:results} describe the evaluation and findings, followed by the
scope of the conclusions in Sections~\ref{sec:limitations}
and~\ref{sec:conclusion}.

\section{Related Work}
\label{sec:related}

\paragraph{Conformal classification and VLMs.}
Common conformal classification scores include the least-ambiguous set-valued
classifier (LAC) score, adaptive prediction sets (APS), and regularized adaptive
prediction sets (RAPS). These scores convert classifier probabilities into
finite-sample-valid label sets with different efficiency and adaptivity
properties \citep{romano2020classification,angelopoulos2021uncertainty}. A
systematic evaluation of conformal vision and vision-language foundation models
shows that the backbone, score, and adaptation strategy all affect coverage and
efficiency \citep{fillioux2026are}. Conf-OT adapts zero-shot logits through a
transductive optimal-transport procedure \citep{silva2025zeroshot}, whereas
localized VLM conformal prediction varies thresholds using feature-space
similarity around a query \citep{fuchs2026localized}. Empirical-Bayes
$r$-values incorporate variability across score realizations
\citep{zeng2026empirical}, and test-time augmentation can improve conformal
efficiency through a separately learned augmentation policy
\citep{shanmugam2025tta}. These methods address set efficiency, transfer, or
local adaptation. In contrast, we study coverage within the groups induced by
adaptive source acquisition.

\paragraph{Multiple predictors and adaptive computation.}
Score averaging and symmetric aggregation exploit multiple predictors while
preserving conformal validity \citep{ochoa2024ensemble,luo2025averaging,
alami2026symmetric}, but these approaches ordinarily acquire every constituent
output. Cascaded conformal inference instead prunes candidates across
increasingly expensive models \citep{fisch2020cascade}, and nested early-exit
prediction sets address consistency across intermediate exits
\citep{jazbec2024early}. Conformal Cascade uses conformal set size to defer
large language model (LLM) queries and provides a general $1-K\alpha$
accepting-tier guarantee, which becomes tighter under an additional
selection-preservation condition \citep{dou2026cascade}. Conformal
learning-to-defer quantifies uncertainty in a rejector that allocates decisions
between a model and a human expert \citep{fang2026learning}. \sagecp addresses
a complementary target: it returns a label set from the final adaptively
aggregated VLM evidence. The selection rule is frozen before calibration, and
the induced route groups are protected directly rather than through a per-stage
union bound.

\paragraph{Mondrian and conditional coverage.}
Unrestricted distribution-free pointwise conditional coverage is impossible
without strong assumptions \citep{barber2021limits}. Finite partitions permit
meaningful group guarantees, and several methods extend or learn such groups
\citep{vovk2012conditional,gibbs2025conditional,martinezgil2024groups,
bairaktari2025kandinsky}.
Hierarchical pooling can also stabilize class-conditional conformal prediction
when infrequent labels make classwise calibration data scarce
\citep{papaioannou2026robust}. Our finite-sample scarcity problem is analogous,
but the groups are computational routes rather than semantic classes.
Mondrian CP calibrates separately within the categories of a fixed taxonomy
\citep{vovk2005algorithmic}. Once the router is frozen, its routes form such a
taxonomy, and standard split Mondrian calibration produces exactly the same
thresholds and sets as \sagecp. Proposition~\ref{prop:mondrian-equivalence}
formalizes this identity. The conformal quantile itself is standard; the new
statistical target is the adaptive route. We characterize why this operational
partition matters, how severely marginal coverage can fail within it, and how
the routing policy can be learned without contaminating final calibration.

Unlike demographic strata or feature clusters, a route is part of the
predictor rather than side information about an input: it determines both the
acquired models and the returned score. This coupling makes route coverage a
direct deployment target when reliability is required along acquisition paths.
Routes are not universally superior to other partitions; a different fixed
taxonomy can protect a different declared population. In every case, the
development/calibration boundary is essential because the route, score, and
taxonomy must be fixed before the final conformal quantiles are fitted.

\section{Problem Formulation}
\label{sec:problem}

Let $\labels=\{1,\ldots,L\}$ and let $p_k(y\mid x)$ denote the probability
assigned to label $y$ by frozen source $k\in\{1,\ldots,K\}$. A routing policy
selects a source sequence and a stopping depth $R(x)$. It returns an aggregated
probability $\bar p_{R(x)}(\cdot\mid x)$ and nonconformity score
\begin{equation}
 A(x,y)=1-\bar p_{R(x)}(y\mid x).
 \label{eq:routed-score}
\end{equation}
Eq.~\eqref{eq:routed-score} is the score actually returned by the adaptive
system. We separate data used to design this system, $\caldev$, from the final
conformal calibration split $\calcal$. Test labels are used only for metrics.

The router induces an observable group $G(x)\in\mathcal G$, where
$\mathcal G$ is a finite route taxonomy. Our primary construction uses
$\mathcal G=\{0,1\}$, with $G(x)=0$ for a first-source exit and $G(x)=1$ after
escalation. For any prediction set $\widehat\Gamma(X)$, define the prevalence
and coverage of group $g$ as
\begin{align}
 \pi_g&=\Pr\{G(X)=g\},\\
 c_g&=\Pr\{Y\in\widehat\Gamma(X)\mid G(X)=g\}.
\end{align}

Marginal coverage controls only the weighted average
$\sum_g\pi_gc_g$ and therefore does not determine how errors are divided
between routes. Proposition~\ref{prop:marginal-route} characterizes the weakest
coverage that remains possible within any one route.

\subsection{The blind spot of marginal validity}

\begin{proposition}[Sharp marginal-to-route bound]
\label{prop:marginal-route}
If $\Pr\{Y\in\widehat\Gamma(X)\}\geq1-\alpha$, then every group with
$\pi_g>0$ satisfies
\begin{equation}
 c_g\geq\left[1-\frac{\alpha}{\pi_g}\right]_+.
 \label{eq:sharp-floor}
\end{equation}
The bound is attainable for every $\alpha$ and $\pi_g$.
\end{proposition}

\begin{proof}
All other group coverages are at most one, hence
\[
1-\alpha\leq\sum_h\pi_hc_h\leq\pi_gc_g+1-\pi_g.
\]
Rearrangement and $c_g\geq0$ give~\eqref{eq:sharp-floor}. For sharpness, if
$\pi_g>\alpha$, set $c_g=1-\alpha/\pi_g$ and set all complementary coverage to
one. If $\pi_g\leq\alpha$, set $c_g=0$ and complementary coverage to
$(1-\alpha)/(1-\pi_g)\leq1$. A binary-label problem with an independent
uniform variable controlling inclusion realizes either construction.
\end{proof}

\begin{corollary}[Sharp binary route-gap bound]
For two groups with prevalences $\pi$ and $1-\pi$, let
$s=\min\{\pi,1-\pi\}$. Marginal $1-\alpha$ validity alone permits
\begin{equation}
 \sup|c_0-c_1|=\min\!\left\{1,\frac{\alpha}{s}\right\}.
 \label{eq:sharp-gap}
\end{equation}
\end{corollary}

\begin{proof}
If group $g$ has lower coverage, Proposition~\ref{prop:marginal-route} gives
$1-c_g\leq\min\{1,\alpha/\pi_g\}\leq\min\{1,\alpha/s\}$. Equality follows by
assigning coverage one to the complement of a group with prevalence $s$, and
coverage $1-\alpha/s$ when $s>\alpha$ or zero when $s\leq\alpha$ to that group.
\end{proof}

\textbf{The consequence is stark. A route containing at most an $\alpha$
fraction of inputs can have zero coverage while the complete system remains
marginally valid.} Even a $25\%$ route at $\alpha=0.10$ has only a $60\%$
distribution-free floor. Eq.~\eqref{eq:sharp-gap} gives the corresponding
binary-route disparity.

\section{Method}
\label{sec:method}

\subsection{Intuition}

The bound reveals what went wrong: the guarantee averaged over the routes.
The repair is to make those routes part of the calibration target. This needs
a clean separation of roles. Development data design the routing and
aggregation policy. Final calibration data provide coverage. Once the policy
is frozen, route membership is observable without the test label. We can then
calibrate a threshold for the score distribution on each operational path.
In short: develop first, freeze once, and calibrate by route.

Routing and calibration solve different problems. Routing decides which
evidence is acquired and creates the compute savings. Mondrian calibration is
not unique to routes; it protects whichever fixed, covariate-only groups define
the reliability target. We use routes because early-exit and escalated inputs
are the operational populations created by the deployed predictor.

\subsection{Development-learned sequential routing}

We greedily order sources on $\caldev$. For a source set $\mathcal S$, define
its correct-label negative log likelihood
\begin{equation}
 \mathcal L(\mathcal S)=-\frac{1}{|\caldev|}
 \sum_{(x_i,y_i)\in\caldev}
 \log\!\left[\frac{1}{|\mathcal S|}
 \sum_{j\in\mathcal S}p_j(y_i\mid x_i)\right].
\end{equation}
Given selected sources $\mathcal S$, the next source is
$k^*=\arg\min_{k\notin\mathcal S}\mathcal L(\mathcal S\cup\{k\})$.
Labels are used only on $\caldev$.

At depth $m$, let $\bar p_m$ be the mean probability across acquired sources
and define the label-free confidence margin
\begin{equation}
 M_m(x)=\bar p_{m,(1)}(x)-\bar p_{m,(2)}(x),
\end{equation}
where $(1)$ and $(2)$ are the largest two entries. Among development examples
that reach depth $m$, the threshold $\tau_m$ is the $q$-quantile of $M_m$.
The next source is queried when $M_m(x)<\tau_m$.

The query fraction $q$ is selected without final calibration or test data.
Three internal development partitions evaluate
$q\in\{.10,.20,.25,.30,.40,.50\}$. A candidate must reach at least $0.88$
overall and worst-route coverage at the primary $0.90$ target. It must also
leave at least 50 calibration examples in each route. Among eligible
candidates, the rule first minimizes source calls and then set size. If none
is eligible, it first maximizes worst-route coverage. The selected $q$, source
order, and thresholds are refit on all of $\caldev$ and then frozen.

The $0.88$ guard is an engineering rule for choosing a compute--set-size
operating point. It is not the conformal guarantee, and it does not reuse the
final calibration labels. Across 50 dataset--outer-split fits, the selector
chose $q=.10$ 28 times, $.20$ 19 times, $.25$ twice, and $.30$ once. It never
chose $.40$ or $.50$. Thus, $q$ is not a benchmark-wide constant. Theorem~
\ref{thm:arbitrary-route} permits any of these development-side choices because
the final calibration and test data remain untouched until the complete policy
is frozen.

\subsection{Route-conditional calibration}

For the main binary partition,
\begin{equation}
 G(x)=\mathbf{1}\!\left\{R(x)>1\right\}.
 \label{eq:binary-route}
\end{equation}
Here, $\mathbf{1}\{\cdot\}$ is the indicator function: it equals one when the
condition is true and zero otherwise. Hence, Eq.~\eqref{eq:binary-route}
assigns first-source exits to group zero and every escalated input to group one.
Let $n_g$ be the number of final calibration examples assigned to route $g$.
For miscoverage level $\alpha$, set
\begin{equation}
 k_g=\left\lceil(n_g+1)(1-\alpha)\right\rceil,
 \qquad
 \widehat q_g=A^{(g)}_{(k_g)},
\end{equation}
where $A^{(g)}_{(k_g)}$ is the $k_g^{\mathrm{th}}$ within-route calibration
order statistic and equals $+\infty$ if $k_g=n_g+1$. The returned set is
\begin{equation}
 \widehat\Gamma_\alpha(x)=
 \{y\in\labels:A(x,y)\leq\widehat q_{G(x)}\}.
 \label{eq:prediction-set}
\end{equation}

Algorithm~\ref{alg:sagecp} summarizes the complete development, calibration,
and prediction pipeline.

\begin{algorithm}[H]
\centering
\caption{\sagecp learns and freezes the routing policy before the final
route-conditional calibration step.}
\label{alg:sagecp}
\fbox{\begin{minipage}{0.94\columnwidth}
\small
\textbf{SAGE-CP}\par\smallskip
\textbf{Input:} frozen sources; independent development and calibration data;
target $\alpha$.\par
\textbf{Development phase}
\begin{enumerate}[leftmargin=1.4em,itemsep=0pt,topsep=2pt]
    \item Select $q$ on internal development partitions.
    \item Greedily order sources and fit margin thresholds.
    \item Freeze the complete route, aggregation score, and group map.
\end{enumerate}
\textbf{Calibration phase}
\begin{enumerate}[leftmargin=1.4em,itemsep=0pt,topsep=2pt]
    \item Route every calibration image without using its label.
    \item Compute the true-label routed score in each route group.
    \item Store the finite-sample conformal quantile $\widehat q_g$ per group.
\end{enumerate}
\textbf{Prediction:} route $x$, form its routed score, and return
$\{y:A(x,y)\leq\widehat q_{G(x)}\}$.
\end{minipage}}
\end{algorithm}

Algorithm~\ref{alg:sagecp} has one central invariant: once final calibration
begins, no part of the router, score, or route map changes.

\subsection{Validity after learned routing}

\begin{proposition}[Exact Mondrian equivalence]
\label{prop:mondrian-equivalence}
Fix a finite route map $H$, returned score $S$, calibration sample, and
$\alpha$. A split Mondrian conformal predictor whose taxonomy is $H$ and whose
nonconformity score is $S$ returns exactly the thresholds and prediction sets
in Eq.~\eqref{eq:prediction-set}.
\end{proposition}

\begin{proof}
Both procedures sort the same calibration scores separately inside each value
of $H$ and use the same finite-sample order statistic
$\lceil(n_h+1)(1-\alpha)\rceil$. Their thresholds, and hence their sets for
every test input, coincide pointwise.
\end{proof}

Proposition~\ref{prop:mondrian-equivalence} locates the contribution precisely.
With a fixed route map, \sagecp calibration is split Mondrian CP. The
routing-specific work comes before the quantile: expose the operational
failure, learn a useful sequential policy, and freeze the paths that define
the taxonomy.

\begin{theorem}[Validity after development-learned routing]
\label{thm:arbitrary-route}
Let $\caldev$ be independent of calibration and test examples. Conditional on
$\caldev$, let $H_{\caldev}:\mathcal X\to\{1,\ldots,J\}$ be any finite
covariate-only grouping induced by an arbitrary learned routing policy, and let
$S_{\caldev}:\mathcal X\times\labels\to\mathbb R$ be its returned score.
Suppose the $n$ calibration examples and a new example are exchangeable
conditional on $\caldev$. Using the within-group finite-sample quantile in
Eq.~\eqref{eq:prediction-set}, every group with positive test probability
satisfies, for almost every $\caldev$,
\begin{equation}
\Pr\!\left\{Y_{\rm new}\in\widehat\Gamma_\alpha(X_{\rm new})
\mid\caldev,H_{\caldev}(X_{\rm new})=h\right\}\geq1-\alpha.
\end{equation}
\end{theorem}

\begin{proof}
Conditioning on $\caldev$ fixes the source operations, stopping policy, score,
and group map. For a fixed group $h$, condition further on all calibration
group indicators and on the new example belonging to $h$. Because the group
map is fixed and covariate-only, the $n_h$ selected calibration true-label
scores remain exchangeable with the new true-label score. The new score has
probability at least $k_h/(n_h+1)\geq1-\alpha$ of not exceeding the
$k_h^{\mathrm{th}}$ calibration order statistic; ties can only increase this
probability. When $k_h=n_h+1$, the threshold is $+\infty$ and coverage is one.
Averaging over the calibration group indicators completes the proof.
\end{proof}

\begin{corollary}[Free development-side policy selection]
The source family, order, stage count, stopping thresholds, aggregation rule,
score family, and finite route partition may be selected with or without labels
on $\caldev$. No multiplicity correction is required for validity if the
selected policy is frozen before final calibration and group membership does
not use the unknown test label.
\end{corollary}

\begin{proof}
Conditioning on $\caldev$ and any development-side random seed fixes the
selected policy, so Theorem~\ref{thm:arbitrary-route} applies directly.
\end{proof}

The theorem covers binary early/escalated groups, full exit-depth calibration,
and any coarser finite partition selected on development data. Greater
resolution has a finite-sample cost. A group threshold is nontrivial only when
$\lceil(n_h+1)(1-\alpha)\rceil\leq n_h$, equivalently
$n_h\geq1/\alpha-1$. At $\alpha=.10$, a group with fewer than nine calibration
examples must receive the $+\infty$ threshold and hence the full label set.
This is the exact point where an overly fine taxonomy becomes vacuous. Larger
groups can still yield variable or inefficient thresholds when their sample is
small. The theorem does not provide pointwise, class-conditional, or
distribution-shift coverage.

\section{Experimental Design}
\label{sec:experiments}

\subsection{Datasets, sources, and splits}

We evaluate ten image-classification datasets spanning objects, textures, remote sensing, fine-grained recognition, food, and scenes. In each dataset, every test image belongs to a fixed set of candidate classes known at inference time. The collection
includes Caltech-101 \citep{feifei2004learning}, the Describable Textures
Dataset (DTD) \citep{cimpoi2014describing}, EuroSAT
\citep{helber2019eurosat}, the Fine-Grained Visual Classification Aircraft
benchmark (FGVC-Aircraft) \citep{maji2013fine}, Flowers-102
\citep{nilsback2008automated}, Food-101 \citep{bossard2014food}, Oxford-IIIT
Pets \citep{parkhi2012cats}, Stanford Cars \citep{krause2013cars}, and SUN397
\citep{xiao2010sun}. The 1,000-class ImageNetV2 matched-frequency set provides
the largest label space in the evaluation \citep{recht2019imagenetv2}.
Table~\ref{tab:data} reports the exact split sizes.

\begin{table}[H]
\centering
\scriptsize
\caption{\textbf{Benchmark split sizes.} Datasets with at least 10,000
available images use a fixed, approximately class-stratified 10,000-image pool
partitioned into 4,000 development, 3,000 calibration, and 3,000 test examples.
Smaller datasets use their complete official pool with an approximately
40/30/30 split.}
\label{tab:data}
\begin{tabular}{lrrrr}
\toprule
Dataset & Classes & Development & Calibration & Test \\
\midrule
Caltech-101 & 101 & 3,470 & 2,603 & 2,604 \\
DTD & 47 & 2,256 & 1,692 & 1,692 \\
EuroSAT & 10 & 4,000 & 3,000 & 3,000 \\
FGVC-Aircraft & 100 & 4,000 & 3,000 & 3,000 \\
Flowers-102 & 102 & 3,275 & 2,456 & 2,458 \\
Food-101 & 101 & 4,000 & 3,000 & 3,000 \\
Oxford Pets & 37 & 2,939 & 2,204 & 2,206 \\
Stanford Cars & 196 & 4,000 & 3,000 & 3,000 \\
SUN397 & 397 & 4,000 & 3,000 & 3,000 \\
ImageNetV2 & 1,000 & 4,000 & 3,000 & 3,000 \\
\bottomrule
\end{tabular}
\end{table}

The heterogeneous source pool contains the RN50-QuickGELU/OpenAI and
ViT-B/32/LAION-2B checkpoints
\citep{radford2021learning,cherti2023openclip}, ViT-L/14/DataComp-XL
\citep{gadre2023datacomp}, and SigLIP2-SO400M/256
\citep{tschannen2025siglip2}. Every source uses the same four-prompt text
ensemble and remains frozen. The development selector places SigLIP2 first on
all splits of nine datasets and ViT-L/14 first on Caltech-101.

The first evaluation split uses a fixed, approximately class-stratified
partition, assigning 40\% of the examples to development, 30\% to calibration,
and 30\% to testing. For datasets with 10,000-image evaluation pools, this
produces 4,000 development, 3,000 calibration, and 3,000 test examples. We
generate four additional approximately class-stratified partitions using
prespecified random seeds and evaluate every method on the same five
partitions. Because these partitions overlap, Student's $t$ intervals computed
from paired differences across partitions are interpreted as descriptive
stability intervals rather than confidence intervals based on five independent
test samples. For the first partition, we additionally compute 2,000 paired
bootstrap replicates over the test examples.

\subsection{Comparisons and metrics}

The primary controlled comparison isolates the calibration boundary while
holding the router, routed probability vectors, calibration budget, and source
calls fixed. Global calibration fits one threshold across all routed examples,
whereas \sagecp fits a separate threshold within each route. Any difference in
route reliability therefore comes from where calibration is performed, not
from a different predictor or compute allocation.

The single-source comparisons use LAC, APS, and RAPS
\citep{romano2020classification,angelopoulos2021uncertainty}. We also evaluate
first-source route LAC, each individual backbone, and the full four-source
probability ensemble. Additional comparisons include Conf-OT with all three
scores \citep{silva2025zeroshot}, localized VLM conformal prediction
\citep{fuchs2026localized}, empirical-Bayes $r$-value CP
\citep{zeng2026empirical}, and a Bonferroni-valid multiclass adaptation of
Conformal Cascade \citep{dou2026cascade}. All primary methods use the same outer
calibration split, and no method receives test labels or a larger calibration
budget in the equal-calibration comparison.

Aggregate efficiency comparisons require matched coverage. We define the
near-target regime as macro coverage within $0.001$ of $0.90$. Methods outside
this band remain informative for coverage and route diagnostics, but their set
sizes are reported separately rather than ranked against near-target methods.

We report marginal coverage, mean set size, average source calls, and coverage
within the two \sagecp route groups. Minimum route coverage ($\mincov$) is the
smaller of the two empirical group coverages, and the route gap is their
absolute difference. For methods not calibrated on these route groups, both
quantities are diagnostics rather than guarantees. Because the minimum of two
noisy estimates is downward biased, Theorem~\ref{thm:arbitrary-route} protects
each fixed group rather than the expectation of their empirical minimum.

All main results use $\alpha=0.10$. We also test
$\alpha\in\{0.05,0.10,0.20\}$, calibration sizes 250, 500, 1,000, and the full
dataset-specific split. For routing sensitivity, we freeze each
$q\in\{.10,.20,.25,.30,.40,.50\}$ in turn and evaluate the nested route
taxonomies $\{1,2+\}$, $\{1,2,3+\}$, and $\{1,2,3,4\}$ on all ten datasets and
five splits. Here, $m+$ denotes every stopping depth greater than or equal to
$m$: for example, $\{1,2+\}$ separates first-source exits from all escalated
inputs, while $\{1,2,3+\}$ retains depth two and pools depths three and four.
These analyses use the same final splits, but neither final labels nor test
metrics select $q$, the taxonomy, a model, a route, a score family, or a
hyperparameter.

\section{Results}
\label{sec:results}

\subsection{RQ1: Does marginal validity conceal route failure?}

Figure~\ref{fig:main} contrasts the reassuring marginal average with the route
behavior it conceals. Although marginal coverage remains near the target on
every dataset, a global threshold can concentrate most errors in one route.
Escalated images are undercovered on several datasets, whereas first-source
exits are weaker on FGVC-Aircraft. EuroSAT and FGVC-Aircraft provide contrasting
cases because their global route gaps are already small.

\begin{figure}[H]
\centering
\begin{subfigure}[t]{0.60\textwidth}
    \centering
    \includegraphics[width=\linewidth]{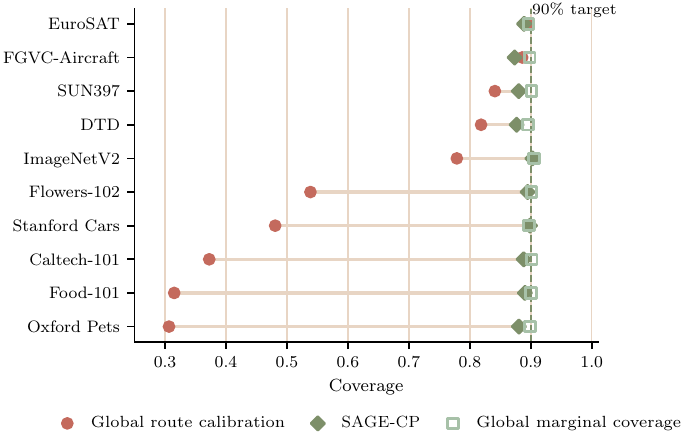}
    \subcaption{\textbf{Marginal validity can hide route failure.}}
    \label{fig:main-datasets}
\end{subfigure}\hfill
\begin{subfigure}[t]{0.35\textwidth}
    \centering
    \includegraphics[width=\linewidth]{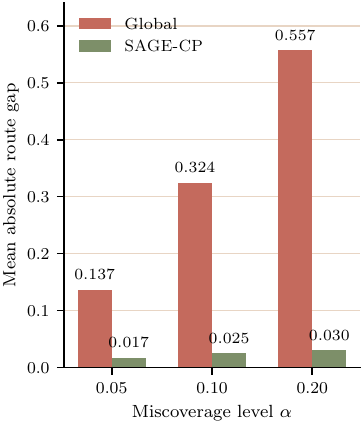}
    \subcaption{\textbf{The repair persists across risk levels.}}
    \label{fig:main-risk}
\end{subfigure}
\caption{\textbf{Global validity does not imply route reliability.}
(\subref{fig:main-datasets}) Both methods operate near $0.90$ marginal coverage,
yet global post-routing calibration has worst-route coverage as low as $0.306$.
\sagecp uses identical routed scores and compute but calibrates the two routes
separately. (\subref{fig:main-risk}) The mean route gap of global calibration
grows with $\alpha$; the \sagecp gap stays at or below $0.030$. Values average
five aligned splits.}
\label{fig:main}
\end{figure}

\FloatBarrier

Table~\ref{tab:primary} reports the corresponding dataset-level comparison
under identical routed predictions and source calls.

\begin{table}[H]
\centering
\small
\caption{\textbf{Equal-score, equal-compute comparison at $\alpha=0.10$.}
Global and \sagecp use the same routed probability vectors and source calls.
$\mathrm{Cov.}$ denotes marginal coverage and $\mathrm{MinCov}$ denotes the
smaller route coverage. $\Delta$ is \sagecp minus global $\mathrm{MinCov}$.
$\dagger$ marks a
descriptive paired five-split 95\% stability interval excluding zero; the
repartitions overlap and this is not an independent-replication significance
test.}
\label{tab:primary}
\resizebox{\textwidth}{!}{%
\begin{tabular}{lrrrrrrrrrr}
\toprule
& \multicolumn{4}{c}{Global route calibration} &
\multicolumn{5}{c}{\sagecp} & \\
\cmidrule(lr){2-5}\cmidrule(lr){6-10}
Dataset & Cov. & MinCov & Gap & Size & Cov. & MinCov & Gap & Size & Calls & $\Delta$ MinCov \\
\midrule
Caltech-101 & .902 & .373 & .611 & .937 & .901 & .888 & .031 & .999 & 1.198 & +.516$^\dagger$ \\
DTD & .895 & .818 & .094 & 2.575 & .898 & .877 & .042 & 2.733 & 1.227 & +.058$^\dagger$ \\
EuroSAT & .897 & .895 & .011 & 4.424 & .896 & .889 & .021 & 4.424 & 1.186 & -.006 \\
FGVC-Aircraft & .898 & .886 & .018 & 3.093 & .898 & .874 & .035 & 3.098 & 1.113 & -.012 \\
Flowers-102 & .901 & .539 & .439 & 1.044 & .898 & .895 & .009 & 2.623 & 1.218 & +.356$^\dagger$ \\
Food-101 & .900 & .315 & .671 & .916 & .904 & .890 & .020 & 1.089 & 1.169 & +.575$^\dagger$ \\
Oxford Pets & .899 & .306 & .683 & .910 & .896 & .881 & .030 & .979 & 1.183 & +.574$^\dagger$ \\
Stanford Cars & .897 & .481 & .509 & .916 & .902 & .899 & .010 & .983 & 1.243 & +.418$^\dagger$ \\
SUN397 & .901 & .841 & .067 & 1.724 & .901 & .880 & .028 & 1.771 & 1.128 & +.039$^\dagger$ \\
ImageNetV2 & .905 & .779 & .141 & 2.393 & .905 & .903 & .022 & 3.418 & 1.109 & +.124$^\dagger$ \\
\midrule
Macro mean & .900 & .623 & .324 & 1.893 & .900 & .887 & .025 & 2.212 & 1.177 & +.264 \\
\bottomrule
\end{tabular}}
\end{table}

Changing only the calibration boundary has a large effect:
Table~\ref{tab:primary} shows that \sagecp raises macro $\mincov$ by 26.4
percentage points and reduces the route gap by 92.3\%, while marginal coverage
and source calls remain unchanged. The paired split interval for the $\mincov$
difference excludes zero on eight datasets and includes zero on EuroSAT and
FGVC-Aircraft. Since the partitions overlap, these intervals describe the
stability of the direction rather than independent-sample significance.

The stronger route target also changes set efficiency. Paired stability
intervals indicate larger \sagecp sets on six datasets and include zero on the
remaining four. The macro mean rises from 1.893 to 2.212 labels. Flowers-102
shows the largest change, from 1.044 to 2.623 labels, while its route gap falls
from 0.439 to 0.009. Thus, the primary gain is route reliability rather than a
uniform reduction in set size relative to global calibration.

\paragraph{Why calibrate the routes?}
We isolate operational alignment from generic stratification with a matched
control. Using development data only, we choose a threshold on the first
source's maximum probability that matches the prevalence of the escalated
route. We then freeze this confidence taxonomy and calibrate the same routed
LAC score within its two bins. The router, final calibration examples, source
calls, and routed probability vectors remain unchanged, so the comparison
changes only which groups receive separate conformal quantiles.
Table~\ref{tab:taxonomy-alignment} compares the two taxonomies.

\begin{table}[H]
\centering
\small
\caption{\textbf{Operational routes versus a matched confidence taxonomy at
$\alpha=0.10$.} Ten-dataset means over five aligned splits. Both Mondrian
methods have 0.900 coverage and 1.177 source calls. Neither calibration
taxonomy dominates: each is strongest on its own groups, while route-Mondrian
returns smaller sets.}
\label{tab:taxonomy-alignment}
\setlength{\tabcolsep}{2.2pt}
\begin{tabular}{lrrrrr}
\toprule
& \multicolumn{2}{c}{Actual routes} &
\multicolumn{2}{c}{Confidence bins} & \\
\cmidrule(lr){2-3}\cmidrule(lr){4-5}
Calibration & MinCov & Gap & MinCov & Gap & Set size \\
\midrule
Global & .623 & .324 & .592 & .359 & 1.893 \\
Route-Mondrian & .887 & .025 & .845 & .064 & 2.212 \\
Confidence-Mondrian & .883 & .028 & .891 & .018 & 2.442 \\
\bottomrule
\end{tabular}
\end{table}

The two partitions agree on 95.5\% of test images, so similar route behavior is
expected. Route calibration has slightly higher actual-route minimum coverage,
a slightly smaller route gap, and sets that are 0.230 labels smaller on
average. The route-minus-confidence paired
differences are $+.005\pm.006$ for actual-route MinCov,
$-.003\pm.007$ for the route gap, and $-.230\pm.126$ labels for set size;
these are descriptive 95\% stability intervals because the outer splits
overlap. The two route statistics are therefore modest relative to split
variability. Conversely, confidence calibration is better within the bins it
defines. Together, these results confirm that calibration protects its declared
groups. Routes are the primary target here because they are the populations
created by, and exposed to, the system's acquisition decisions; the confidence
taxonomy is a calibration control rather than a replacement for the sequential
policy that produces the routed score and compute savings.

\takeaway{A nominal global average can coexist with near-worst-case route
behavior. Route calibration repairs this failure without changing the router
or its compute. A prevalence-matched confidence taxonomy gives a strong but
different guarantee. Routing supplies adaptive compute; calibration determines
which operational population is protected.}

\subsection{RQ2: Does route reliability require full inference?}

Route reliability does not require four source evaluations. The full ensemble
reaches $0.889$ $\mincov$, compared with $0.887$ for \sagecp, and both have a
$0.025$ route gap. However, \sagecp uses 1.177 sources instead of four and
returns smaller sets on the macro mean, 2.212 versus 2.486 labels. Across
datasets, its $\mincov$ is tied with the ensemble on nine datasets and lower on
one, while its sets are descriptively smaller on seven.

Table~\ref{tab:baselines} separates near-target comparisons from methods that
overcover. Conf-OT APS and RAPS have similar diagnostic route balance. \sagecp
has descriptively smaller sets on six of ten datasets than Conf-OT APS and on
seven than Conf-OT RAPS. It uses 0.177 more source calls on average. Localized
VLM-CP and Conf-OT LAC often keep the route disparity despite their small sets.
The Bonferroni cascade covers at $0.972$, far above the $0.90$ target. Its set
size is therefore not a matched-coverage efficiency comparison.

The route-calibrated first-source baseline also achieves comparable route
coverage with one call, but its mean sets are larger than those of \sagecp.
This comparison separates the two components of the design: route calibration
supplies reliability, while selective escalation acquires additional evidence
when the first source is insufficient.

\takeaway{Route reliability does not require full inference. \sagecp matches
the full ensemble's group behavior with 70.6\% fewer source calls, while the
first-source baseline exposes the smaller compute--informativeness tradeoff.}

\begin{table}[H]
\centering
\small
\caption{\textbf{Coverage-aligned comparisons and unmatched diagnostics at
$\alpha=0.10$.}
Values are unweighted means over ten datasets and five aligned splits.
Panel A contains methods within 0.001 of the 0.90 target. The primary
matched comparison isolates the effect of route-conditional calibration
under identical scores, routing, and compute. The remaining rows provide
broader coverage-aligned comparisons. Panel B reports methods that
materially overcover; their set sizes and source counts are diagnostic
and should not be ranked against Panel A.}
\label{tab:baselines}

\begin{tabular}{llrrrrr}
\toprule
Regime & Method & Coverage & MinCov & Gap & Set size & Sources \\
\midrule

\multicolumn{7}{l}{%
    \textit{Panel A: Coverage-aligned efficiency comparisons}}\\
\midrule

\multicolumn{7}{l}{\textit{Primary matched comparison}}\\
\hdashline
Equal score/compute
    & Global routed LAC
    & .900 & .623 & .324 & 1.893 & 1.177 \\
Proposed
    & \sagecp{} (route-Mondrian)
    & .900 & .887 & .025 & 2.212 & 1.177 \\

\hdashline
\multicolumn{7}{l}{%
    \textit{Other coverage-aligned baselines}}\\
\hdashline
Single source
    & Vanilla first-source LAC
    & .900 & .610 & .339 & 1.958 & 1.000 \\
Single source
    & Vanilla first-source APS
    & .901 & .884 & .030 & 3.456 & 1.000 \\
Single source
    & Vanilla first-source RAPS
    & .901 & .882 & .032 & 3.008 & 1.000 \\
Single source/route CP
    & First-source route LAC
    & .900 & .893 & .019 & 2.338 & 1.000 \\
Transductive adaptation
    & Conf-OT LAC
    & .899 & .622 & .324 & 1.670 & 1.000 \\
Transductive adaptation
    & Conf-OT APS
    & .901 & .891 & .015 & 2.669 & 1.000 \\
Transductive adaptation
    & Conf-OT RAPS
    & .900 & .888 & .019 & 2.453 & 1.000 \\
Localized calibration
    & Localized VLM-CP LAC
    & .899 & .611 & .338 & 1.888 & 1.000 \\
Full compute
    & Four-source route ensemble
    & .900 & .889 & .025 & 2.486 & 4.000 \\

\midrule
\multicolumn{7}{l}{%
    \textit{Panel B: Coverage-unmatched diagnostics---not ranked by efficiency}}\\
\midrule

Full compute
    & Empirical-Bayes $r$-value CP
    & .921 & .734 & .218 & 3.240 & 4.000 \\
Adaptive cascade
    & Conformal Cascade adaptation
    & .972 & .908 & .076 & 6.749 & 2.604 \\
\bottomrule
\end{tabular}
\end{table}

\subsection{RQ3: Is the conclusion robust?}

We first vary the target risk level while keeping the routing and evaluation
protocol fixed. Table~\ref{tab:alpha} summarizes the resulting macro means.

\begin{table}[H]
\centering
\small
\caption{\textbf{Risk-level sensitivity.} Ten-dataset macro means. The global
route gap worsens as the target becomes more aggressive; \sagecp remains near
balanced.}
\label{tab:alpha}
\begin{tabular}{rrrrrrr}
\toprule
$\alpha$ & \multicolumn{3}{c}{Global} & \multicolumn{3}{c}{\sagecp} \\
\cmidrule(lr){2-4}\cmidrule(lr){5-7}
& Cov. & MinCov & Gap & Cov. & MinCov & Gap \\
\midrule
.05 & .950 & .834 & .137 & .950 & .942 & .017 \\
.10 & .900 & .623 & .324 & .900 & .887 & .025 \\
.20 & .798 & .327 & .557 & .801 & .784 & .030 \\
\bottomrule
\end{tabular}
\end{table}

The route disparity is not specific to $\alpha=.10$. As $\alpha$ rises from
0.05 to 0.20, global $\mincov$ falls from 0.834 to 0.327 and its gap grows from
0.137 to 0.557. Over the same range, the \sagecp gap remains between 0.017 and
0.030.

We next vary the number of calibration examples to examine the finite-sample
effect predicted by the route-specific quantiles. Table~\ref{tab:calsize}
reports the comparison.

\begin{table}[H]
\centering
\small
\caption{\textbf{Calibration-size sensitivity at $\alpha=0.10$.} Small route
samples make \sagecp conservative and variable, but do not recreate the severe
global disparity.}
\label{tab:calsize}
\resizebox{10cm}{!}{%
\begin{tabular}{rrrrrrrr}
\toprule
Calibration & \multicolumn{3}{c}{Global} & \multicolumn{4}{c}{\sagecp} \\
\cmidrule(lr){2-4}\cmidrule(lr){5-8}
$n$ & Cov. & MinCov & Gap & Cov. & MinCov & Gap & Size \\
\midrule
250 & .905 & .641 & .310 & .909 & .892 & .044 & 2.845 \\
500 & .902 & .630 & .320 & .904 & .887 & .033 & 2.311 \\
1,000 & .900 & .624 & .324 & .901 & .889 & .026 & 2.239 \\
Full & .900 & .623 & .324 & .900 & .887 & .025 & 2.212 \\
\bottomrule
\end{tabular}
}
\end{table}

The results follow the expected finite-sample pattern. With 250 calibration
examples, small route strata produce more conservative sets and a 0.044 mean
gap. At 1,000 examples and at the full calibration size, the gap decreases to
0.026 and 0.025, respectively, while global calibration remains imbalanced at
every budget.

The next analysis refines the route taxonomy from the primary early/escalated
partition to individual exit depths. Table~\ref{tab:route-depth} quantifies the
resulting trade-off between route resolution and calibration support.

\begin{table}[H]
\centering
\small
\caption{\textbf{Exit-depth granularity at $\alpha=0.10$.} Unweighted
ten-dataset means over the same five splits and development-selected policies.
``Smallest $n_g$'' is the mean size of the smallest calibration stratum. A vacuous
split contains at least one route assigned the exact $+\infty$ threshold.}
\label{tab:route-depth}
\resizebox{\textwidth}{!}{%
\begin{tabular}{lrrrrrrrrr}
\toprule
& & \multicolumn{2}{c}{Global calibration} &
\multicolumn{4}{c}{Route-conditional calibration} & & \\
\cmidrule(lr){3-4}\cmidrule(lr){5-8}
Taxonomy & Smallest $n_g$ & MinCov & Gap & Coverage & MinCov & Gap & Set size & Vacuous splits & Vacuous test \\
\midrule
$\{1,2+\}$       & 389.8 & .623 & .324 & .900 & .887 & .025 & 2.212 & 0/50  & .000\% \\
$\{1,2,3+\}$     &  65.5 & .435 & .514 & .900 & .872 & .062 & 2.269 & 0/50  & .000\% \\
$\{1,2,3,4\}$    &  12.7 & .373 & .586 & .900 & .849 & .126 & 2.414 & 27/50 & .061\% \\
\bottomrule
\end{tabular}}
\end{table}

As the taxonomy becomes finer, global macro $\mincov$ falls from 0.623 to
0.373, whereas separate calibration repairs much of the disparity. The mean
smallest calibration stratum simultaneously decreases from 389.8 examples to
65.5 and then 12.7. The three-group taxonomy remains nonvacuous on every split.
For the four-depth taxonomy, at least one route has fewer than nine examples in
27 of 50 splits and therefore receives the exact full-set threshold. These
routes account for only 0.061\% of test images on the macro average (87 of
134,800 pooled predictions, or 0.065\% under pooled weighting), although they
increase set size and group variability. The primary binary partition therefore
provides the most useful balance between operational resolution and calibration
support in this evaluation.

The rare-route cases are spread across the benchmark rather than confined to
one dataset. Table~\ref{tab:vacuity-datasets} gives the full breakdown.
FGVC-Aircraft and ImageNetV2 are vacuous in all five splits, while DTD,
Flowers-102, and Stanford Cars are vacuous in one. In every case, the affected
test population is small.

\begin{table}[H]
\centering
\scriptsize
\caption{\textbf{Per-dataset four-depth vacuity at $\alpha=0.10$.}
``Splits'' counts outer splits with at least one exact $+\infty$ threshold.
``Test'' is the mean fraction of test images assigned to vacuous routes.}
\label{tab:vacuity-datasets}
\begin{tabular}{lrr}
\toprule
Dataset & Splits & Test (\%) \\
\midrule
Caltech-101    & 3/5 & .1075 \\
DTD            & 1/5 & .0118 \\
EuroSAT        & 2/5 & .0800 \\
FGVC-Aircraft  & 5/5 & .1067 \\
Flowers-102    & 1/5 & .0081 \\
Food-101       & 3/5 & .0733 \\
Oxford Pets    & 2/5 & .0363 \\
Stanford Cars  & 1/5 & .0267 \\
SUN397         & 4/5 & .1000 \\
ImageNetV2     & 5/5 & .0600 \\
\midrule
Macro mean     & 27/50 & .0610 \\
\bottomrule
\end{tabular}
\end{table}

\FloatBarrier

To separate the calibration result from the development-selected query
fraction, Table~\ref{tab:q-sensitivity} freezes each candidate value of $q$
before final calibration.

\begin{table}[H]
\centering
\scriptsize
\caption{\textbf{Fixed query-fraction sensitivity at $\alpha=0.10$.}
Ten-dataset macro means; each row freezes $q$ before final calibration.}
\label{tab:q-sensitivity}
\resizebox{10cm}{!}{%
\begin{tabular}{rrrrrrrr}
\toprule
& & \multicolumn{2}{c}{Global} & \multicolumn{4}{c}{\sagecp} \\
\cmidrule(lr){3-4}\cmidrule(lr){5-8}
$q$ & Calls & MinCov & Gap & Coverage & MinCov & Gap & Size \\
\midrule
.10 & 1.110 & .537 & .405 & .900 & .888 & .027 & 2.176 \\
.20 & 1.247 & .691 & .262 & .900 & .894 & .015 & 2.202 \\
.25 & 1.328 & .735 & .220 & .902 & .895 & .015 & 2.217 \\
.30 & 1.417 & .764 & .193 & .901 & .894 & .014 & 2.205 \\
.40 & 1.622 & .804 & .157 & .901 & .895 & .013 & 2.207 \\
.50 & 1.871 & .832 & .134 & .901 & .896 & .011 & 2.207 \\
\bottomrule
\end{tabular}}
\end{table}

Across the full grid, \sagecp macro coverage remains 0.900--0.902,
$\mincov$ remains 0.888--0.896, and the route gap remains 0.011--0.027.
Increasing $q$ queries more sources and makes the two route populations less
imbalanced even under global calibration, but the global gap remains
0.134--0.405. Development selection therefore chooses the compute operating
point, whereas route-conditional calibration produces the reliability repair.

The source pool is genuinely heterogeneous. Table~\ref{tab:backbones} shows
that SigLIP2 is usually strongest, but not always. ViT-L/14 is selected first
on Caltech-101, and EuroSAT gives it a smaller set than SigLIP2. Even a strong
first source is not sufficient on every image, which leaves room for selective
escalation.

\begin{table}[H]
\centering
\scriptsize
\caption{Mean global LAC set size for each frozen source at $\alpha=0.10$.
Near-nominal coverage holds for all source columns. Bold indicates the smallest
set per dataset; the final column gives the development-selected first source.}
\label{tab:backbones}
\resizebox{10cm}{!}{%
\begin{tabular}{lrrrrl}
\toprule
Dataset & RN50 & ViT-B/32 & ViT-L/14 & SigLIP2 & First source \\
\midrule
Caltech-101 & 1.493 & .997 & \textbf{.936} & .958 & ViT-L/14 \\
DTD & 15.176 & 7.699 & 4.173 & \textbf{2.707} & SigLIP2 \\
EuroSAT & 6.359 & 4.208 & \textbf{3.500} & 4.655 & SigLIP2 \\
FGVC-Aircraft & 33.651 & 22.099 & 5.679 & \textbf{3.369} & SigLIP2 \\
Flowers-102 & 14.067 & 7.507 & 1.674 & \textbf{1.018} & SigLIP2 \\
Food-101 & 1.831 & 1.692 & .938 & \textbf{.919} & SigLIP2 \\
Oxford Pets & 1.255 & 1.088 & .928 & \textbf{.915} & SigLIP2 \\
Stanford Cars & 4.174 & 1.106 & .937 & \textbf{.916} & SigLIP2 \\
SUN397 & 4.144 & 2.581 & 1.865 & \textbf{1.745} & SigLIP2 \\
ImageNetV2 & 14.953 & 12.776 & 4.167 & \textbf{2.403} & SigLIP2 \\
\bottomrule
\end{tabular}
}
\end{table}

\paragraph{Empty-set sensitivity.}
LAC produces empty sets on five high-accuracy datasets. A deployment
sensitivity adds the routed top-1 label only when the conformal set is empty.
This deterministic enlargement cannot reduce coverage and requires no tuned
parameter. At $\alpha=0.10$, it removes all empty sets but raises \sagecp macro
coverage from 0.900 to 0.936 and its route gap from 0.025 to 0.057. The raw
conformal construction therefore remains primary; forced non-empty output is
an explicitly conservative deployment option, not an improved calibration
method.

\takeaway{The route-reliability result persists across three risk levels,
calibration budgets, every fixed query fraction, ten datasets, and four
heterogeneous sources. Finer groups remain valid but expose the predicted
finite-sample resolution cost; the non-empty sensitivity likewise shows why
seemingly benign post-processing must be reported rather than silently applied.}

\FloatBarrier
\subsection{RQ4: Do source savings survive deployment?}

Source-call reductions are useful only when they translate into lower
deployment cost. Figure~\ref{fig:deployment} and Table~\ref{tab:deployment}
show that they do: in the higher-budget FGVC-Aircraft and ImageNetV2 runs,
\sagecp reduces latency by 59.4\% and 61.0\% relative to full inference. These
reductions correspond to $2.46\times$ and $2.56\times$ speedups. Gross GPU
energy decreases by 54.4\% and 57.6\%, while estimates adjusted for resident
idle power decrease by 50.0\% and 53.7\%.

The routing overhead itself is small, averaging 0.062\,ms per image on
FGVC-Aircraft and 0.060\,ms per image on ImageNetV2. As an implementation check, end-to-end
sequential execution reproduces the routing decisions used in the offline
evaluation for all 1,536 examined inference cases. The mean absolute difference
between probabilities computed in the two settings remains below
$3.5\times10^{-8}$, and none of the measured images requires evaluation of the
fourth source.

Mean first-source latency is slightly higher than \sagecp latency, but the
paired intervals include zero, consistent with decode and cache variation being
larger than this small difference. The methods are therefore indistinguishable
in this comparison. All methods also occupy approximately 7 GiB because every
checkpoint remains resident, so the measured deployment benefit is lower
latency and GPU energy than full inference rather than lower resident memory.

One reduced-budget protocol extends the measurement to five domains in
Table~\ref{tab:deployment-expanded}. The result is consistent: \sagecp uses
1.062--1.109 sources per image, cuts full-inference latency by
58.8\%--63.4\%, and cuts gross GPU energy by 55.7\%--57.5\%. Every runtime
routing depth matches the frozen score archive across 1,280 inferences. The
largest probability discrepancy is $1.8\times10^{-7}$, and routing takes only
0.048--0.065 ms per image. This common reduced-budget protocol evaluates
cross-domain consistency, while the 256-image, three-repeat runs in
Table~\ref{tab:deployment} provide the higher-budget estimates for
FGVC-Aircraft and ImageNetV2.

\takeaway{Across five measured datasets, adaptive routing preserves most
full-ensemble route reliability and cuts full-inference latency by about 60\%.
The statistical routing repair adds negligible decision overhead but does not
reduce resident model memory.}

\section{Limitations and Scope}
\label{sec:limitations}

Our guarantee assumes that final calibration and test examples are
exchangeable conditional on development data. ImageNetV2 adds 1,000 labels and
a replicated ImageNet collection, but all three splits still come from
ImageNetV2. It is not a calibration-under-shift experiment. When calibration
and deployment distributions differ, a shift-aware method such as weighted
conformal prediction is needed \citep{tibshirani2019shift}.

Route-conditional validity is neither pointwise nor class conditional.
Per-class coverage remains poor on some fine-grained benchmarks, as it does for
the baselines. The primary binary target trades resolution for calibration
sample size. Our three- and four-depth study confirms both sides of that
trade-off. Rare fourth exits trigger exact full-set thresholds in 27 of 50
splits. They cover only 0.061\% of evaluated test images, but finer or
intersectional taxonomies can be inefficient without more calibration data.
The exact vacuity threshold is $n_h<1/\alpha-1$.

\begin{figure}[H]
\centering
\begin{subfigure}[t]{0.48\textwidth}
    \centering
    \includegraphics[width=\linewidth]{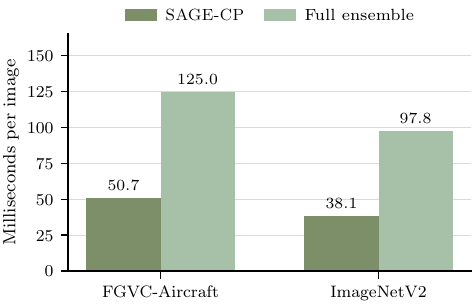}
    \subcaption{\textbf{End-to-end sequential latency.}}
    \label{fig:deployment-latency}
\end{subfigure}
\hfill
\begin{subfigure}[t]{0.48\textwidth}
    \centering
    \includegraphics[width=\linewidth]{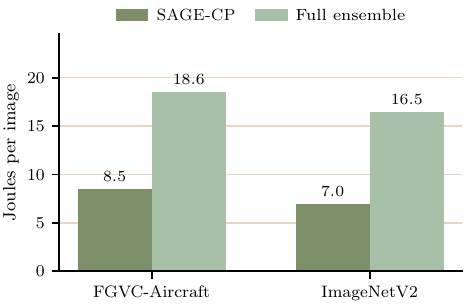}
    \subcaption{\textbf{Sampled gross GPU energy.}}
    \label{fig:deployment-energy}
\end{subfigure}
\caption{\textbf{True batch-one deployment cost on one RTX A4500.}
Measurements include real-image loading and decode, model-specific
preprocessing, host-to-device transfer, synchronized encoder inference and
softmax, routing, and Python orchestration. All four checkpoints and text
embeddings remain resident. Energy is estimated from sampled NVIDIA board
power.}
\label{fig:deployment}
\end{figure}


\begin{table}[H]
\centering
\small
\caption{\textbf{Primary higher-budget sequential deployment benchmark.}
Means use 256 test images
and three repeats. Paired differences compare \sagecp with full inference;
intervals are 95\% Student-$t$ intervals across repeats. GPU energy is sampled
board telemetry rather than whole-system energy. Latency is measured in
milliseconds per image, energy in joules per image, and memory in mebibytes
(MiB).}
\label{tab:deployment}
\resizebox{\textwidth}{!}{%
\begin{tabular}{llrrrrrr}
\toprule
Dataset & Method & Latency & Throughput & Sources & Gross energy & Peak memory & Paired diff. \\
& & ms/image & images/s & & J/image & MiB & ms/image \\
\midrule
\multirow{3}{*}{FGVC-Aircraft}
& First source & 52.71 & 18.97 & 1.000 & 8.07 & 6,994 & -- \\
& \sagecp & 50.70 & 19.73 & 1.121 & 8.47 & 6,994 & $-74.26\pm1.06$ \\
& Full ensemble & 124.95 & 8.00 & 4.000 & 18.56 & 6,994 & -- \\
\midrule
\multirow{3}{*}{ImageNetV2}
& First source & 40.24 & 24.85 & 1.000 & 6.95 & 7,006 & -- \\
& \sagecp & 38.14 & 26.22 & 1.098 & 6.99 & 7,006 & $-59.63\pm12.22$ \\
& Full ensemble & 97.77 & 10.23 & 4.000 & 16.47 & 7,006 & -- \\
\bottomrule
\end{tabular}}
\end{table}

\begin{table}[H]
\centering
\small
\caption{\textbf{Cross-domain deployment extension under one common reduced
budget.} Each method uses the same deterministic 128-image subset and two
repeats. The paired difference is \sagecp minus full latency with a 95\%
Student-$t$ interval across the two repeats. Energy is sampled gross NVIDIA
board energy; percentage columns compare \sagecp with full inference.}
\label{tab:deployment-expanded}
\resizebox{\textwidth}{!}{%
\begin{tabular}{lrrrrrrrrr}
\toprule
& \multicolumn{5}{c}{End-to-end latency} & \multicolumn{3}{c}{Gross GPU energy} & \\
\cmidrule(lr){2-6}\cmidrule(lr){7-9}
Dataset & \sagecp & Full & Paired diff. & Reduction & Speedup & \sagecp & Full & Reduction & Calls \\
& \multicolumn{2}{c}{ms/image} & ms/image & \% & $\times$ & \multicolumn{2}{c}{J/image} & \% & \\
\midrule
DTD           & 37.70 &  91.85 & $-54.15\pm48.29$ & 59.0 & 2.44 & 6.81 & 15.57 & 56.3 & 1.109 \\
FGVC-Aircraft & 45.98 & 125.54 & $-79.56\pm 5.16$ & 63.4 & 2.73 & 7.65 & 17.88 & 57.2 & 1.102 \\
Food-101      & 35.08 &  92.37 & $-57.29\pm 3.02$ & 62.0 & 2.63 & 6.48 & 15.24 & 57.5 & 1.062 \\
ImageNetV2    & 40.21 &  97.59 & $-57.38\pm26.87$ & 58.8 & 2.43 & 7.19 & 16.22 & 55.7 & 1.102 \\
SUN397        & 50.86 & 130.60 & $-79.75\pm15.70$ & 61.1 & 2.57 & 8.33 & 19.04 & 56.2 & 1.109 \\
\midrule
Macro mean    & 41.97 & 107.59 & --                & 60.8 & 2.56 & 7.29 & 16.79 & 56.6 & 1.097 \\
\bottomrule
\end{tabular}}
\end{table}




Our empirical study focuses on classification tasks with a fixed candidate label set known at inference time. Extending the evaluation to open-vocabulary detection, retrieval, and generation requires task-specific candidate spaces and definitions of semantic equivalence. The source pool covers heterogeneous architectures and training recipes within the dual-encoder image--text model family. Finally, because the repeated partitions overlap, their intervals summarize stability across partitions rather than independent replications.

We benchmark batch-one inference on a single RTX A4500 with all four models
resident, which isolates the computational savings from adaptive source
execution without introducing model-loading overhead. Absolute latency may vary
with the serving stack, batch size, compilation strategy, offloading policy,
and hardware. Energy measurements use sampled GPU board power and therefore do
not include CPU or whole-system consumption. Because all models remain
resident, adaptive routing reduces inference latency and GPU energy but does
not reduce model-residency memory. The five-dataset deployment extension uses
two repeats to assess whether the observed savings persist across domains under
the same hardware configuration. For the remaining five datasets, we report
source calls as a compute indicator rather than a direct latency measurement.

The prevalence-matched confidence taxonomy provides a strong non-routing
control and agrees with the operational route partition on 95.5\% of test
images, which explains their similar actual-route statistics. Other fixed,
label-free partitions can produce different efficiency and reliability
trade-offs. The claim is therefore operational: route calibration protects the
acquisition paths of the deployed predictor, while another taxonomy is
appropriate when a different population is the stated target.

Finally, \sagecp can enlarge prediction sets when global calibration allocates
errors unevenly. That is the statistical cost of the stronger group target.
When the routed system is already balanced, as on EuroSAT and Aircraft, the
empirical benefit is small.

\section{Conclusions}
\label{sec:conclusion}

An adaptive system creates more than a prediction. It also creates operational
populations: the inputs that stop early and those that require more evidence.
A marginal guarantee can be correct for the system and badly wrong for one of
these populations.

\sagecp addresses this mismatch through a calibration-aligned design: the routing policy is learned on development data and frozen before route-specific conformal thresholds are estimated from independent calibration data. Across ten datasets, this procedure moves mean worst-route
coverage from 0.623 to 0.887 without changing marginal coverage or routed
compute. It matches the route behavior of the full ensemble with 70.6\% fewer
source calls and smaller macro-mean sets. Compared with global calibration, its
sets are larger on six datasets as the cost of the stronger target. On five
measured datasets, it also reduces full-inference latency by
58.8\%--63.4\%.

The value of the method lies in this simplicity. Once adaptive inference makes
its routes observable, those routes can be protected without retraining the
predictor or surrendering its compute savings. The router decides how much
evidence to acquire. Calibration decides where the guarantee holds. When the
deployment target is the early-exit and escalated populations, the guarantee
should follow those routes rather than only their average.

\subsubsection*{Broader Impact Statement}

\sagecp is designed to reveal and correct route-dependent reliability failures
in adaptive prediction systems. Route-level coverage diagnostics help determine
whether early-exit and escalated inputs receive comparable protection, while
adaptive source acquisition can reduce the computational and energy costs of
evaluating a full model ensemble. However, route-conditional validity does not
by itself provide individual-level, class-conditional, demographic-group, or
distribution-shift guarantees. Its finite-sample guarantee also relies on
exchangeability between the final calibration data and deployment inputs. In
high-stakes applications, prediction sets should therefore support, rather than
replace, application-specific validation and human judgment. Responsible
deployment can combine \sagecp with relevant subgroup diagnostics, monitoring
of route prevalence and distribution shift, shift-aware or class-conditional
calibration when appropriate, and domain-expert oversight for consequential
decisions.

\bibliographystyle{tmlr}
\bibliography{refs}

\end{document}